\documentclass{article}

\usepackage{microtype}
\usepackage{url}
\usepackage{caption}
\usepackage{graphicx}
\usepackage{amsmath}
\usepackage{arydshln}
\usepackage{amsthm,amsfonts}
\usepackage{booktabs}
\usepackage{algorithm}
\usepackage{algorithmic}
\usepackage[switch]{lineno}

\usepackage{xcolor}
\usepackage{balance} 
\usepackage{subcaption}
\usepackage{comment}

\usepackage{multirow}

\usepackage{pifont}

\newtheorem{definition}{Definition}[section]

\newtheorem{proposition}{Proposition}[section]

\usepackage{trimclip}
\usepackage{soul}
\usepackage{caption}
\usepackage{hyperref}
\hypersetup{colorlinks,allcolors=black}

\usepackage{pifont}
\usepackage{xspace}

\definecolor{darkgreen}{RGB}{0,150,0} 
\definecolor{darkred}{RGB}{200,0,0}   

\newcommand{\greencheck}{{\color{darkgreen} \ding{51}}} 
\newcommand{\redx}{{\color{darkred} \ding{55}}} 

\newcommand{\spectra}{SPECTra\xspace}

\newcommand{\shortto}{\clipbox*{{.50\width} 0pt {\width} {\height}} \textrightarrow}

\begin{document}

\title{SPECTra: Scalable Multi-Agent Reinforcement Learning with Permutation-Free Networks}

\author{%
\textbf{Hyunwoo Park}$^{1}$, Baekryun Seong$^{1}$, Sang-Ki Ko$^{1}$ \\
$^1$University of Seoul \\
}





\maketitle

\begin{abstract}
In cooperative multi-agent reinforcement learning (MARL), the permutation problem—where the state space grows exponentially with the number of agents—reduces sample efficiency. Additionally, many existing architectures struggle with scalability, relying on a fixed structure tied to a specific number of agents, limiting their applicability to environments with a variable number of entities. While approaches such as graph neural networks (GNNs) and self-attention mechanisms have progressed in addressing these challenges, they have significant limitations as dense GNNs and self-attention mechanisms incur high computational costs.
To overcome these limitations, we propose a novel agent network and a non-linear mixing network that ensure permutation-equivariance and scalability, allowing them to generalize to environments with various numbers of agents. Our agent network significantly reduces computational complexity, and our scalable hypernetwork enables efficient weight generation for non-linear mixing. Additionally, we introduce curriculum learning to improve training efficiency.
Experiments on SMACv2 and Google Research Football (GRF) demonstrate that our approach achieves superior learning performance compared to existing methods. By addressing both permutation-invariance and scalability in MARL, our work provides a more efficient and adaptable framework for cooperative MARL. Our code is available at \href{https://github.com/funny-rl/SPECTra}{\textbf{\color{blue}https://github.com/funny-rl/SPECTra}}.
\end{abstract}

\section{Introduction}
Reinforcement learning (RL) has achieved remarkable success in various domains such as games~\cite{MnihKSGAWR13,SilverHMGSDSAPL16,VinyalsBCMDCCPE19}, robotics~\cite{ZhaoQW20,KoberBP13}, and natural language processing ~\cite{ChristianoLBMLA17,RafailovSMMEF23, deepseek}. However, traditional RL aims to train a single agent, limiting its applicability to complex problems requiring interaction among multiple agents. To address this limitation, multi-agent reinforcement learning (MARL) has emerged as a framework that enables two or more agents to interact and learn to achieve a common goal. MARL offers significant potential to solve complex real-world challenges, such as dynamic resource allocation~\cite{NguyenKL18}, traffic control~\cite{Bazzan09}, and team sports~\cite{tizero}.

However, MARL faces significant challenges arising from the inherent complexity of multi-agent systems. One major hurdle is the scalability of the learning algorithms to varying numbers of agents. Traditional MARL algorithms using MLPs rely on fixed-size input representations~\cite{qmix}, making them ill-suited for training with dynamic agent populations~\cite{Dyma-CL}. In addition, the exponential growth rate of action space with the number of agents increases the exploration difficulty~\cite{pmic,mat}, forcing agents to take a long time to explore. Furthermore, entities an agent observes are implicitly ordered in sequence, leading to permutation problems where the state space grows exponentially with the number of agents. This reduces sample efficiency~\cite{pic}.

\begin{table}[ht!]
    \centering
    \caption{Comparison of the proposed framework with baselines in terms of two key properties.}
            \begin{tabular}{l  cc}
            \toprule
            \bf{Algorithm}  & \bf{Scalable} & \bf{Permutation-Free}   \\ \midrule
            \textbf{SPECTra-VDN}         & \greencheck       & \greencheck  \\ 
            HPN-VDN          & \redx       & \greencheck  \\
            UPDeT-VDN         & \greencheck        & \greencheck \\ \midrule
            \textbf{SPECTra-QMIX$^+$}       & \greencheck       & \greencheck\\
            HPN-QMIX       & \redx       & \redx  \\ 
            UPDeT-QMIX   & \redx       & \redx \\
            \bottomrule
            \end{tabular}
    \label{tab:algorithms_specifications}
\end{table}

\begin{figure*}[t!] 
\centering
\includegraphics[width=1.\textwidth]{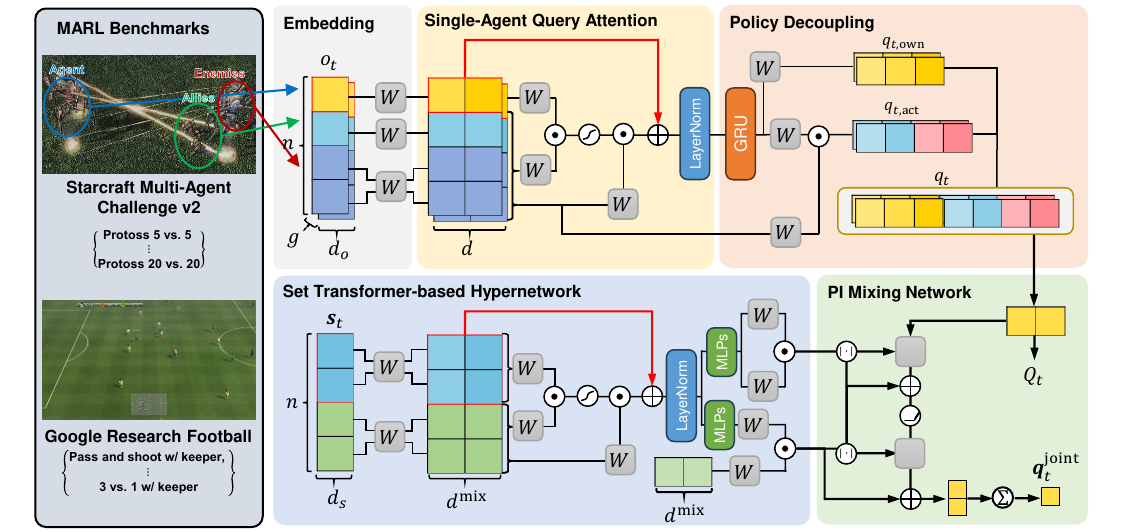} 
\caption{Overview of the proposed SPECTra framework. Gray boxes represent linear layers.}
\label{fig:wide2}
\end{figure*}

To address these issues, we propose the \underline{\bf S}calable and \underline{\bf P}ermutation-free \underline{\bf E}fficient \underline{\bf C}urriculum and \underline{\bf Tra}nsferable learning framework based on \underline{\bf Tra}nsformer architecture, called the SPECTra framework. The agent network and mixing network of \spectra allow the networks to learn independently of the variable number of agents, even with a fixed structure. The implicit ordering of observation and action permutations is solved using permutation equivariance and permutation invariance. As a result, the agent network and mixing network improve sample efficiency compared to traditional methods and can learn from dynamic agent populations.

For the agent network to solve scalability and permutation problems, we propose single-agent query attention (SAQA), which reduces the time complexity to $O(nd)$. SAQA utilizes cross-attention by querying only information relevant to the observer agent. Unlike self-attention, which processes interactions between the observed entities in all directions, this approach focuses only on entities with which the observer interacts, effectively improving scalability~\cite{sample_complexity}. It also utilizes the policy decoupling~\cite{UPDeT} to compute action values independently of the scalability and permutation problems.
We also design a set transformer-based hypernetwork (ST-HyperNet), which is a hypernetwork~\cite{HaDL16,qmix} that generates weight matrices and bias vectors of the mixing network to ensure the scalability and permutation-freeness of the mixing network while also maintaining the monotonicity of the network as in QMIX~\cite{qmix}. The agent network and mixing network can utilize their mathematical property to perform sample-efficient dynamic curriculum learning.

The main contributions of the paper are as follows:
\begin{itemize}
    \item We propose the scalable and permutation-free MARL framework called SPECTra by employing Transformer-based architecture in agent networks and hypernetworks.
    \item We provide a theoretical analysis of the inference time of the proposed framework, especially the permutation freedom and computational efficiency, and a qualitative analysis of the sample complexity of SAQA.
    \item We demonstrate that SPECTra outperforms baselines such as UPDeT and HPN in MARL benchmarks such as SMACv2 and GRF, further showing that our framework is well-suited to curriculum learning and transfer learning settings.
\end{itemize}

\section{Related Work}

\subsection{Multi-Agent Reinforcement Learning}

Many cooperative MARL algorithms have used a centralized training decentralized execution (CTDE)~\cite{CTDE,CTDE2} framework to allow multiple agents to achieve a common goal. CTDE is a training method that calculates centralized value using each agent's extra signals during training, and each agent chooses an action independently using only its observations during execution. VDN~\cite{VDN} and QMIX~\cite{qmix} are representative value-based CTDE algorithms. VDN utilizes additive value decomposition to represent a linear joint-action value function as the sum of individual action values. Linear joint-action value function solves some cooperation problems, such as the lazy agent, but a linear mixing network can represent only a limited class of centralized action-value functions. QMIX significantly improves the representation capability by redesigning the linear joint-action value function of VDN into a non-linear mixing network satisfying Individual-Global-Max (IGM)~\cite{qtran} using hypernetworks. In addition, QMIX with normalized optimization~\cite{qmix_sota} shows that QMIX can outperform QMIX's variants~\cite{qplex,qattan,WQMIX} in SMAC~\cite{SMAC} benchmark through regularized optimization. However, VDN and QMIX still can't solve scalability and permutation problems. 

\subsection{Permutation-Invariant and Equivariant Modeling for Multi-Agent Systems}


There have been many MARL frameworks that utilize permutation equivariance and invariance to address observation and action permutation problems~\cite{HazraDD24,E2GN2,MF-PPO}. Among the frameworks, HPN~\cite{hpn} has demonstrated an effective solution to the problems by independently embedding each observed entity using a permutation-equivariant hypernetwork. As a result, HPN has outperformed existing permutation-free architectures such as Deep Set~\cite{Deepset}, self-attention~\cite{transformer}, and GNNs~\cite{gnn_1,gcn,gnn_survey}. 

However, HPN inherently applies a summation operation for achieving the permutation invariance on independent entity embeddings, preventing it from selectively focusing on specific entities to select the optimal action. Additionally, integrating a non-linear mixing network like QMIX cannot guarantee permutation invariance due to its reliance on simple MLPs, which are not inherently permutation-invariant. Furthermore, both the agent network and the mixing network face challenges in transfer learning and dynamic curriculum learning, as their architectures are dependent on the number of agents.

\subsection{Dynamic Curriculum Learning}

Curriculum learning has emerged as a powerful paradigm for accelerating MARL. Using a staged learning approach, curriculum learning begins with more straightforward tasks and progressively transitions to more complex challenges. In MARL, one standard method for defining task difficulty involves varying the number of agents. Increasing the number of agents inherently introduces greater environmental dynamics and stochasticity, making coordination and collaboration among agents significantly more challenging~\cite{Factorized-q-learning}. For instance, Dyma-CL~\cite{Dyma-CL}, a GNN-based curriculum learning framework, applied curriculum learning to simple agent networks. Their findings demonstrate that even simple algorithms such as IQL~\cite{IQL} and VDN can achieve more efficient training by leveraging curriculum learning by gradually increasing agents instead of training from scratch.
Recent advancements in automatic curriculum learning~\cite{SPC-CL,EPC-CL,VACL} have introduced frameworks that automatically generate and adapt curricula by dynamically assigning the number of agents. These developments underscore the potential of designing models capable of adapting to varying agent counts, motivating further exploration into architectures that remain invariant to the number of agents.

A notable related effort is the UPDeT~\cite{UPDeT} framework, which, while not explicitly incorporating curriculum learning, offers an important example of a universal MARL pipeline. UPDeT employs transformer-based self-attention mechanisms to handle observations and action configurations of dynamic sizes using a unified architecture. This design partially enables zero-shot learning across agent configurations. However, UPDeT's reliance on dense self-attention connections imposes an $O(n^2d)$ computational complexity, and MLPs-based hypernetwork limits UPDeT's scalability of non-linear mixing networks. Moreover, self-attention computations significantly increase sample complexity due to dense connections, which diminishes the benefits of scalability~\cite{sample_complexity}.

\section{Background}

\subsection{Problem Formulation}

We formulate the MARL problem as a decentralized partially observable Markov decision process (Dec-POMDP)~\cite{dec-pomdp,dec-pomdp-book} with a 8-tuple~$M = \langle S, A, U, P, r, O, n, \gamma \rangle$. Let $S$ be the global state of the environment, $A = \{ a_1, \dots, a_n\}$ be the set of $n$ entities, and $U$ is the action space. Since an agent is an independent decision-making entity that interacts with the environment, we denote the set of agents by $A_{\rm agent}$ and the number of agents by $m$. Namely, $|A_{\rm agent}| = m$, where $n > m$.

We denote the {\em entity-wise observation} from entity $a_1$ to entity $a_2$ by $o_{1\shortto2} \in O$. Then, the local observation of entity $a_i$ is denoted by a set of observations $o_i = \{o_{i \shortto j} \mid j \in [n] \}$. 
An agent $a_i$ selects an action $u \in U$ at each time step $t$, and actions selected by $m$ agents form a joint action $\textbf{u} \in \textbf{U} \equiv U^m$.
We suppose that each agent $a_i$ has its own action space $U_i$, which is independent of the permutation of agents.
We also let $P(s' | s, \mathbf{u}): S \times \textbf{U} \times S \to [0, 1]$ denote the state transition probability from $s$ to $s'$ given the joint action~$\mathbf{u}$. The immediate reward function $r(s, \mathbf{u}): S \times \mathbf{U} \to R$, while $\gamma \in [0, 1)$ is a discount factor. At time step $t$, an agent $a_i$ observes a joint observation $o_{i,t}$ and takes an action $u_t \in U_i$ using the policy $\pi_\theta(u_t|O_t)$ parameterized by $\theta$. 
Now we can define the joint action value at time step $t$ following the policy $\pi_\theta$ as follows:
\[
Q^{\pi_\theta}(s_t, \textbf{u}_t) = \mathbb{E}[R_t | s_t, \mathbf{u}_t],
\]
where $R_t = \sum_{i=0}^\infty \gamma^i r_{t+i}$ is the discounted return.

\subsection{Permutation Invariance / Permutation Equivariance}



In MARL, the order of agents can change, and it leads to an explosion in the number of possible input representations due to permutations. In the worst case, a naive approach would require considering all $n!$ possible permutations. Moreover, the output complexity also increases as the joint action space scales as $U^m$.

Let us first formally define the PI and PE attributes from the perspective of MARL as follows:
\begin{definition}[permutation]
    Let $S_n$ be the set of all bijective functions (permutations) from the set $\{1, 2, \ldots, n\}$ to itself.
\end{definition}

\begin{definition}[permutation equivariance]
A function $f : X^n \to Y^n$ is permutation equivariant iff for any permutation $\sigma \in S_n$ and an input matrix $X \in \mathbb{R}^{n \times d}$, $f (P_\sigma X) = P_\sigma f(X),$ where $P_\sigma \in \mathbb{R}^{n \times n}$ is the permutation matrix corresponding to $\sigma$, which acts on $X$ by permitting its rows.
\end{definition}

\begin{definition}[permutation invariance]
A function $f : X^n \to Y^n$ is permutation invariant iff for any permutation $\sigma \in S_n$ and an input matrix $X \in \mathbb{R}^{n \times d}$, $f (P_\sigma X) = f(X)$, where $P_\sigma \in \mathbb{R}^{n \times n}$ is the permutation matrix corresponding to $\sigma$.
\end{definition}

To ensure that the overall framework remains permutation-free, it is not sufficient for the model architecture alone to be permutation-invariant or permutation-equivariant. The structure of observations and states must also meet specific conditions as follows:
\begin{enumerate}
\item First, each agent has a fixed index for encoding its own observation, meaning that the way an agent perceives itself does not change regardless of how other agents are ordered.
\item Second, when choosing an action, an agent can either take an action that directly affects itself (e.g., move) or an action that interacts with another entity (e.g., attack).
\item Lastly, since observations are represented as decomposable vectors, entities that are not directly observable due to partial observability should be masked, ensuring that missing information does not introduce unintended dependencies.
\end{enumerate}

By maintaining these structural constraints, the entire learning process remains permutation-free, leading to more stable and efficient training.

\section{Main Contributions}

This section describes the agent network and mixing network of \spectra in detail. First, we explain the important elements of the agent network: SAQA and action value estimation via policy decoupling. Then, prove that the agent network is scalable and permutation-equivariant. We also use a Set Transformer~\cite{set-transformer}-based hypernetwork called the ST-HyperNet to show that non-linear mixing networks can estimate joint-action values in a permutation-invariant way and prove that the mixing network is scalable and permutation-invariant.

\subsection{Lightweight Attention Mechanism: Single-Agent Query Attention (SAQA)}
\begin{figure}[h!]
    \centering
    \begin{subfigure}{0.3\linewidth}
        \centering
        \includegraphics[width=\textwidth]{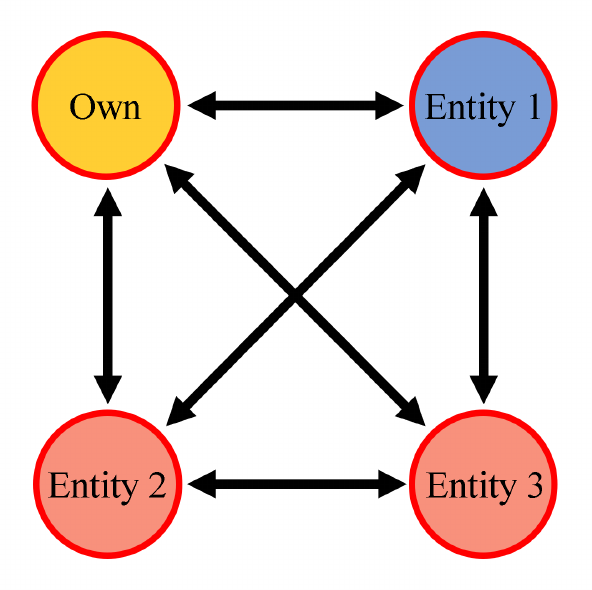}
        \caption{Self-attention}
        \label{fig:dense_connection}
    \end{subfigure}
    \hspace{0.5cm}
    \begin{subfigure}{0.3\linewidth}
        \centering
        \includegraphics[width=\textwidth]{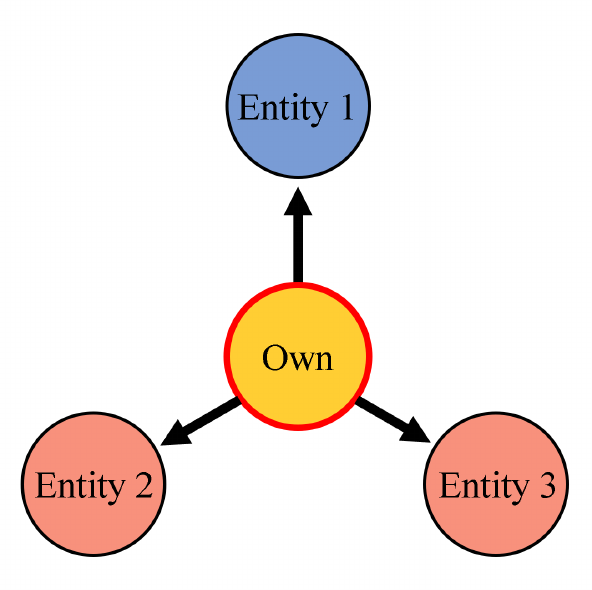}
        \caption{SAQA (ours)}
        \label{fig:sparse_connection}
    \end{subfigure}
    \caption{Comparison of connection density between self-attention and single-agent query attention.}
    \label{fig:attention_comparison}
\end{figure}

To address the high computational cost of the self-attention mechanism, there have been alternative attention mechanisms proposed to reduce its complexity from $O(n^2d)$ to $O(nd^2)$, such as Efficient Attention~\cite{Efficient_attention} and Hydra-Attention~\cite{Hydra-attention}. However, these algorithms have limited applicability because they require the hidden dimension $d$ to be smaller than the number $n$ of agents to be efficient. These limitations highlight the need for more tailored approaches to efficient attention mechanisms in MARL, ensuring computational scalability while preserving critical agent-centric information. 

On the other hand, we propose a lightweight attention mechanism called {\em single-agent query attention (SAQA)}, which performs attention operations with a single agent information as a query and contextualizes the agent's embedding with more relevant observations using the attention mechanism. In this way, we can achieve the computational cost of $O(nd)$ instead of $O(n^2d)$. Figure~\ref{fig:attention_comparison} Shows the difference in structural complexity between self-attention and SAQA. 

Given the set of observations $o_i = \{o_{i \shortto j} \mid j \in A_{i}^{\rm obs}\}$ from agent $a_i$, we first embed the observations via three separate linear layers $E_{\rm own}$, $E_{\rm ally}$, and $E_{\rm enemy}$ follows:
\begin{align*}
e_{i}  = &\; \{ E_{\rm own}(o_{i \shortto i}) \} \cup \{ E_{\rm ally} (o_{i \shortto j}) \mid j \in A_{\rm ally}\} \\
& \cup  \{ E_{\rm enemy} (o_{i \shortto j}) \mid j \in A_{\rm enemy}\} \subset \mathbb{R}^{d_h}.
\end{align*}


Note that $E_{\rm *}(o_{i \shortto j}) \in \mathbb{R}^{d_h}$ is a vector that encodes the observation of entity $a_j$ from agent $a_i$ and defined as $ o_{i \shortto j} \mathbf{W}_*$, where $\mathbf{W}_* \in \mathbb{R}^{d_o \times d_h}$ for $* \in \{{\rm own}, {\rm ally}, {\rm enemy}\}$. Note that `$*$' is determined by the type of entity $a_j$ from $a_i$'s perspective (e.g., $* = \textrm{own}$ if $i = j$). Let us denote $E_{\rm *}(o_{i \shortto j})$ by $e_{i \shortto j}$ for simplicity.

Let $n_h$ be the number of attention heads in each layer. Then, $n_h \times d_k = d_h$, where $d_k$ is the hidden dimension of each attention head. We compute the query, key, and value vector for $k$-th attention head by applying a linear transformation from $d_h$ to $d_k$ as follows:
\[
Q_{i\shortto j}^{(k)}, K_{i\shortto j}^{(k)}, V_{i\shortto j}^{(k)} = e_{i \shortto j} \mathbf{W}_Q^{(k)}, e_{i \shortto j} \mathbf{W}_K^{(k)}, e_{i \shortto j} \mathbf{W}_V^{(k)}.
\]

Now, we compute the attention-weighted vector for agent $i$:
\[
e_{i, {\rm att}}^{(k)}= \text{softmax}\left(\frac{Q_{i\shortto i}^{(k)}(\textbf{K}_{i}^{(k)})^\top}{\sqrt{d_k}}\right)\textbf{V}_{i}^{(k)} \in\mathbb{R}^{d_k }.
\]

We concatenate the vectors from $d_h$ heads as follows:
\[
e_{i, {\rm att}} = 
\begin{bmatrix}
e_{i, {\rm att}}^{(1)} & \dots & 
e_{i, {\rm att}}^{(d_h)}
\end{bmatrix}^\top
\in \mathbb{R}^{d_h}
\]

Lastly, we apply layer normalization after combining the attention-weighted vector with the input query embedding via a residual connection:
\[
e_i' = {\rm LayerNorm}(e_{i \shortto i} + e_{i, {\rm att}}) \in \mathbb{R}^{d_h}.
\]

Remark that the proposed SAQA can be performed with a time complexity of $O(nd)$. We can prove that the entire attention computation is permutation-invariant.

\begin{proposition}\label{prop:saqa_proposition}
    The single-agent query attention of the \spectra framework is a permutation-invariant function.
\end{proposition}


\subsection{Permutation-Free Policy Decoupling} \label{sec:PD}

Remark that the actions against the other entities should be determined in a permutation-equivariant manner, as the input order should not affect the optimal decision-making of the current agent. We can design the policy network of each agent as a permutation-equivariant one by utilizing the permutation-invariant contextualized embedding $e_i'$ from SAQA and the individual entity observations that are also free from the permutation problem.

We improve the policy decoupling proposed by UPDeT~\cite{UPDeT} to output action values permutation-equivariant using linear layers and inner products. Our policy decoupling uses the attention-based embedding of the current agent and the embedding of each observation-entity to compute action values for entities and optimize the policy at an action-group level.

To compute action values based on the action-observation history of agent $a_i$, we employ the GRU cell as follows:
\[
h_{i, t} = \text{GRUCell}(e_i', h_{i, t-1}) \in{\mathbb{R}^{d}}.
\]
From the hidden state $h_{i, t}$, we compute two vectors: one $(Q_{{i,t, {\rm own}}})$ for computing the action value for the actions on the agent itself ($a_i$) and the other $(Q_{{i, t, {\rm act}}})$ for the actions on other agents (including enemies and allies) as follows:
\begin{align*}
    q_{{i, t, {\rm own}}} &=   h_{i,t} \textbf{W}_{\rm own}^Q \in \mathbb{R}^{|U_{\rm own}|}\\
    Q_{{i, t, {\rm act}}}^{(k)} &=  e_i  \textbf{W}_{\rm act}^{(k),Q} \in \mathbb{R}^{d_k} \textrm{ for }1 \le k \le d_h.
\end{align*}

Now we apply a linear transformation to $n$ embedded 
observations as follows:
\[
\textbf{K}_{i, t, \text{act}} = 
\left[
    e_{i \shortto i}^\top; e_{i \shortto 1}^\top;  \dots; e_{i \shortto i-1}^\top;  e_{i \shortto i+1}^\top; \dots; e_{i \shortto n}^\top
\right]
\textbf{W}_{\rm act}^K \in \mathbb{R}^{n \times d_k}.
\]

Now we are ready to compute the $q$-values for the actions that can be performed on the other agents by conducting the dot product operations between a $Q$-vector and $K$-matrix:
\[
q_{i, t, \text{act}} = \sum_{k=1}^{d_h}\frac{Q_{{i, t, {\rm act}}}^{(k)} \left(\textbf{K}_{i, t, \text{act}}^{(k)} \right)^\top}{d_h \sqrt{d_k}} \in \mathbb{R}^{n}.
\]


Note that the $q$-vector is the result of the scaled inner product of $Q_{{i, t, {\rm act}}}^{(k)}$ and $\textbf{K}_{i, t, \text{act}}^{(k)}$. The $K$-matrix represents the latent information of individual observations of all agents. By applying the inner product, the observer agent can estimate the value of taking action on the other agents based on its observations.

Finally, we have the action value of $a_i$ at $t$ as follows:
\begin{align*}
q_{i, t} &= 
[ q_{{i, t, {\rm own}}}; q_{i, t, \text{act}}]^\top \in \mathbb{R}^{|U|}\\
q^{\text{mask}}_{i, t} &= \mathbf{Mask}_{i,t,\text{act}} \odot q_{i,t} + (1 - \mathbf{Mask}_{i,t,\text{act}}) \odot (-\infty),
\end{align*}
where $\mathbf{Mask}_{i,t,\text{act}} \in \{0,1\}^{|U|}$.

We can use the inner product to output $q_{i, t, \text{act}}$ that is permutation-equivariant to the permutations of allies and enemies, and $\textbf{W}_{\rm own}^Q$ to make permutation-independent $q_{{i, t, {\rm own}}}$ permutation-invariant.

\begin{proposition}
The policy decoupling of \spectra framework is a permutation-equivariant function.
\end{proposition}
    
\begin{proof}
    Since proposition~\ref{prop:saqa_proposition}, $e_i'$ is permutation invariant. Thus, we only need to inspect $M_{\text{obs},t}^i$. Expression $e_i$ are permutation equvariant. As other variables are not affected by observation permutation, the proposed policy decoupling method is permutation equivariant.
\end{proof}

\subsection{Set Transformer-based Hypernetwork}
\label{sec:st-hypernet}

Hypernetworks~\cite{HaDL16} refer to a class of neural architectures where one neural network generates the weights of another network. 
To ensure the permutation-equivariance of the mixing network, we employ the Set Transformer~\cite{set-transformer} that consists of the permutation-equivariant encoder and the permutation-invariant decoder.

First we compute the query ${\bf Q}_{i, t}^{W, (k)}$ and key ${\bf K}_{i, t}^{W, (k)}$ values for the agent $a_i$ by applying a linear transformation to the state representation of $a_i$. Then, ST-HyperNet computes the weight matrix and the bias vector of the mixing network as follows:
\begin{align*}
 W^\text{out}_t &= \sum_{k=1}^{d_h} \left|  \frac{{\bf Q}_{i, t}^{W, (k)}\left({\bf K}_{i,t}^{W, (k)} \right)^\top}{d_h \sqrt{d^\text{mix}/ d_h}} \right| \in \mathbb{R}_+^{m \times m}\\
 b^\text{out}_t &= \sum_{k=1}^{d_h}\frac{Q_{i ,t}^{b, (k)} \left({\bf K}_{i, t}^{b, (k)} \right)^\top}{d_h \sqrt{d^{\text{mix}}}} \in \mathbb{R}^{m}
\end{align*}

As in QMIX~\cite{qmix}, absolute value operations are applied to the weights $W^\text{out}_t$ of the mixing network to ensure monotonic value decomposition when computing the joint value function. Additionally, inspired by the decoder of the Set Transformer, we generate the permutation-equivariant bias vector~$ b^\text{out}_t$ using a random seed vector.

\begin{proposition}
The ST-HyperNet of the \spectra framework is a permutation-equivariant function.
\end{proposition}

\subsection{Permutation-Invariant Mixing Network}

Since $W^\text{out}_t, b^\text{out}_t$ obtained using ST-HyperNet are permutation equivariant, we can compute a permutation-invariant joint-action value using the ST-HyperNet as follows:
\begin{align*}
 W^\text{1}_t, b^\text{1}_t &= \text{ST-HyperNet}^1 (S_t) \\
 W^\text{2}_t, b^\text{2}_t &= \text{ST-HyperNet}^2 (S_t) \\
Q_t &= 
\begin{bmatrix}
q_t^1 & \ldots & q_t^m
\end{bmatrix}^\top \in \mathbb{R}^{m} \\
H^{\rm mix}_t &= \text{ReLU}( Q_t W^\text{1}_t + b^\text{1}_t ) \in \mathbb{R}^{m}\\ 
q^{\rm joint}_t &= H^{\rm mix}_t W^\text{2}_t + b^\text{2}_t \in \mathbb{R}
\end{align*}

Simply speaking, since the weights and biases generated by the two ST-HyperNets depend on the order of the agents provided as input, the sequence of $q$-values follows the same order as the weights and biases. This ensures that the computed value of $H_t^{\rm mix}$ remains permutation-equivariant. Finally, the operations such as matrix multiplication and vector addition between $H_t^{\rm mix}$ and the weight matrix and bias vector, which are given in the same permutation, are always guaranteed to be permutation-invariant.

\begin{proposition}
The mixing network of \spectra framework is a permutation-invariant function.
\end{proposition}

\section{Experiments}
We evaluate our proposed SPECTra agent network and SPECTra mixing network on SMACv2 and GRF. We first evaluate our contribution to SMACv2, a successor to SMAC, a benchmark developed to improve the limitations of SMAC's limited stochasticity by using probabilistic unit sampling and initial state distribution. SMACv2 introduces probabilistic unit sampling and diverse initial state distributions, incorporating random team compositions, starting positions, and unit types. This increased stochasticity highlights the necessity of addressing the permutation problem, reinforcing our motivation.

To ensure that SMACv2's training difficulty increases with the number of agents, we turn off the SMACv2 action mask option to reduce the frequency of masked actions. Additionally, to account for the decentralized execution of CTDE, we set the probability of an agent observing an enemy within its sight range to $p = 1.0$. Under these settings, SMACv2 significantly increases training difficulty as the number of agents grows. To smooth the variability of the experimental results and clarify the trend, we applied an exponentially weighted moving average (EWMA) to the experimental results, with a smoothing factor of 0.99. Hyperparameters follow the HPN paper~\cite{hpn}, and detailed hyperparameter settings are described in Appendix~\ref{sec:hyperparameter}.

\subsection{Comparison with Baselines}

\begin{figure*}[t!] 
    \centering
    \includegraphics[width=1.\textwidth]{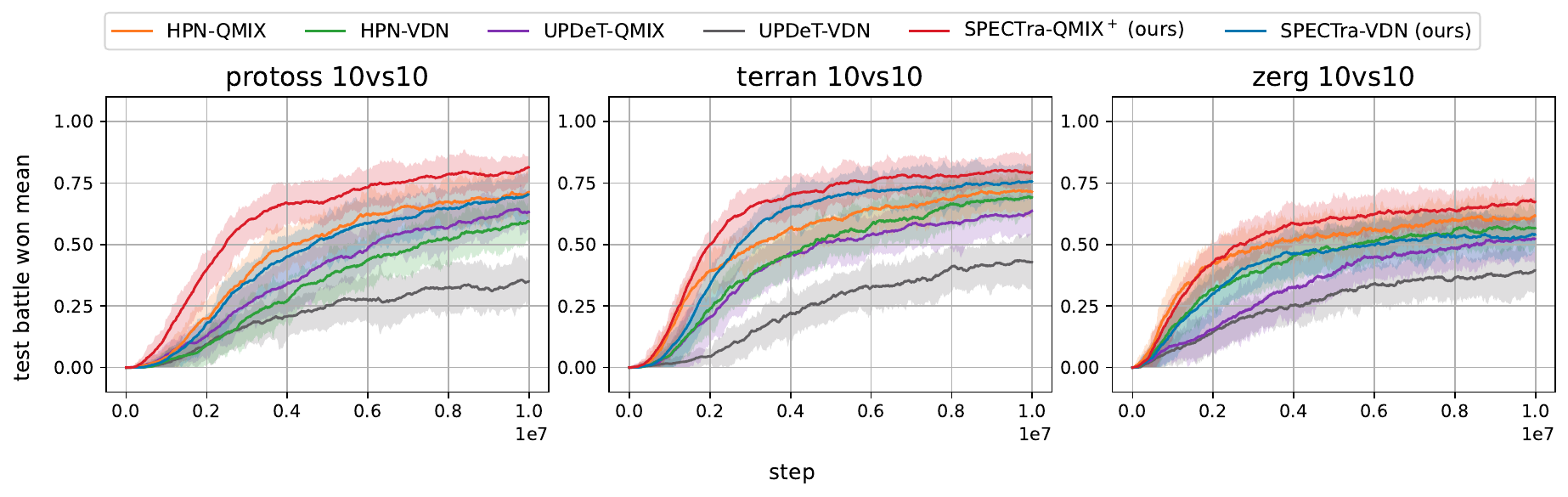}
    \includegraphics[width=1.\textwidth,trim={0 0 0 0cm},clip]{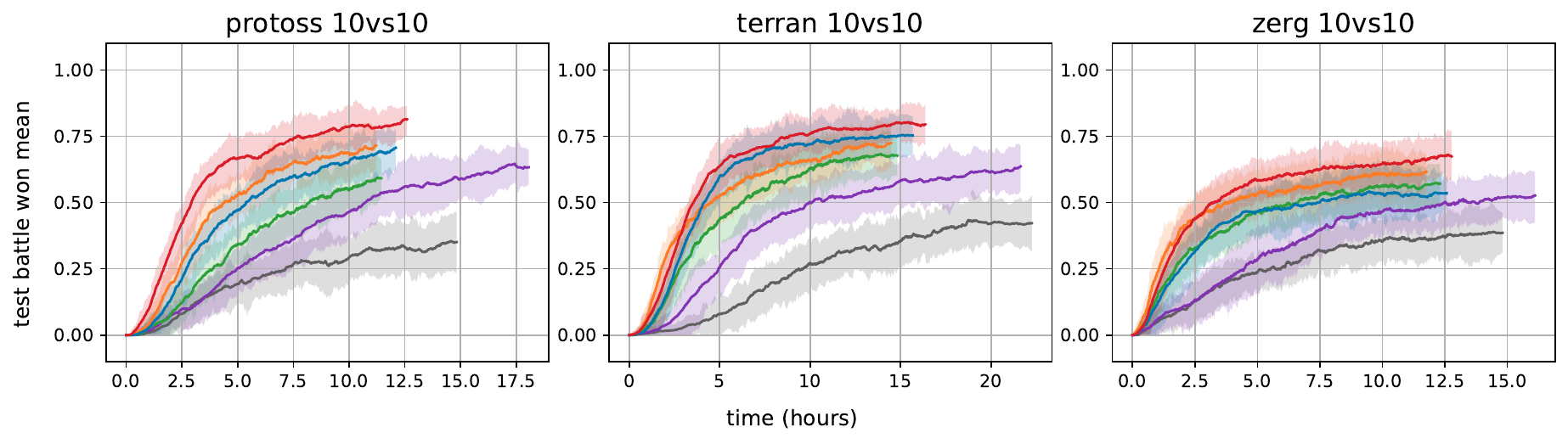} 
    \caption{Comparison of SPECTra-QMIX$^+$ and SPECTra-VDN against baselines in SMACv2.}
    \label{fig:wide}
\end{figure*}

To evaluate the performance of our proposed algorithm, we first compare the performance of HPN, UPDeT, and \spectra based on VDN and QMIX algorithms. The results of the experiments in the 10 vs 10 scenarios of SMACv2 are provided in Figure~\ref{fig:wide}.

The results exhibit that SPECTra-QMIX$^+$ and SPECTra-VDN show strong performance compared to the same QMIX and VDN-based models. SPECTra-QMIX and SPECTra-VDN train much faster than UPDeT-QMIX and UPDeT-VDN utilizing self-attention due to the low computational cost of SAQA. Moreover, as argued in Section~\ref{sec:st-hypernet}, the difference in computation time between QMIX and VDN-based algorithms is not significant, which proves that the amount of computation in the agent network has a significant impact on training time than the amount of computation in the mixing network. The comparison of inference time for each model as the number of agents increases is described in Appendix~\ref{sec:inference_time_comparison}. 

It is important to note that SPECTra-QMIX$^+$ achieved superior performance with significantly fewer parameters than the other baselines. Unlike the baselines, where the number of parameters in the network increases with the number of agents, SPECTra-QMIX$^+$ and SPECTra-VDN do not increase the number of parameters in the network as the number of agents to control increases. The number of parameters in SPECTra-QMIX$^+$ is about $50\%$ of UPDeT-QMIX and $33\%$ of HPN-QMIX, and this difference increases further as the number of agents increases. SPECTra-VDN also has about $48\%$ parameters than HPN-VDN in the above environment, but it performs competitively or is superior to HPN-VDN in most environments. UPDeT-VDN has fewer parameters than SPECTra-VDN, but its scalability is limited due to the sharp rise in computational cost. The parameters for each model can be found in Appendix~\ref{sec:model_parameters}. Meanwhile, despite having more parameters than SPECTra-QMIX$^+$, UPDeT-QMIX does not achieve better performance. The full experimental results can be found in the Appendix~\ref{sec:full_smacv2}. 

\subsection{Curriculum Learning}

\begin{figure*}[ht!]
    \centering
    \includegraphics[width=1.\linewidth]{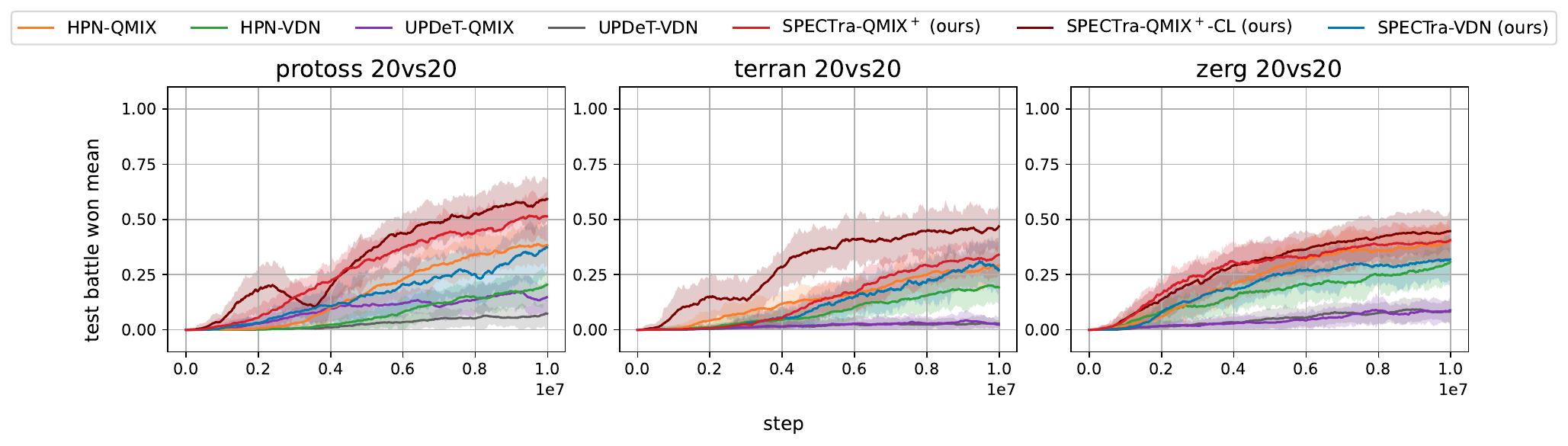}
    \caption{Comparison SPECTra-QMIX$^+$ and SPECTra-VDN with curriculum learning against baselines in SMACV2. }
    \label{fig:curriculum_learning}
\end{figure*}

\begin{figure}[ht!]
    \centering
    \includegraphics[width=0.8\linewidth]{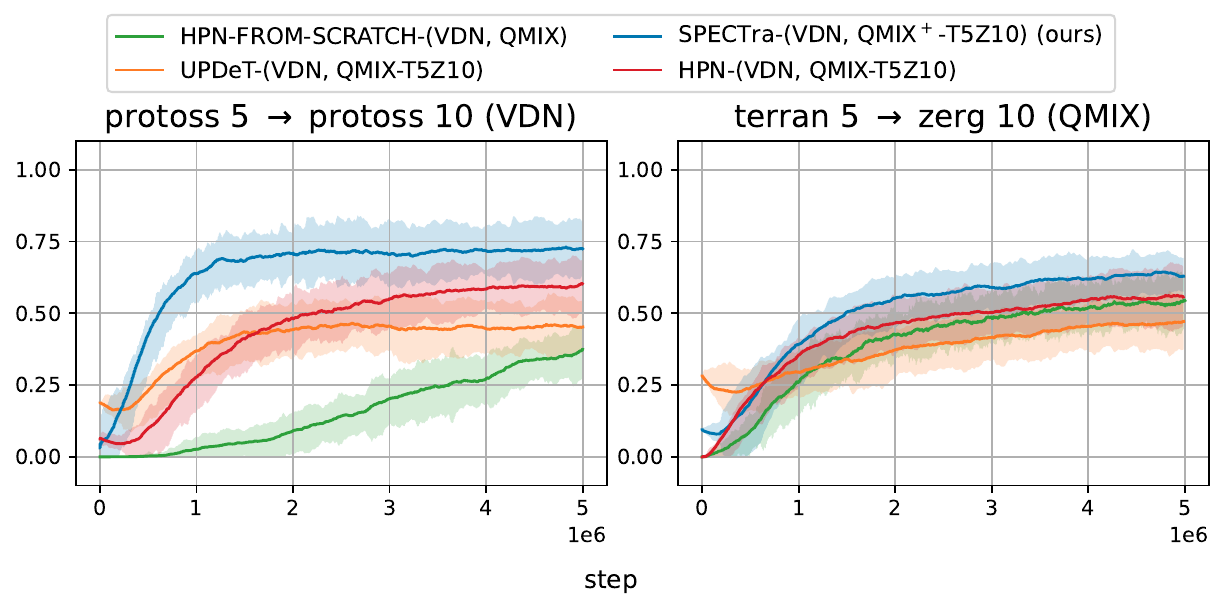}
    \caption{Comparison SPECTra-QMIX$^+$ and SPECTra-VDN with transfer learing against baselines in protoss 5 $\to$ protoss 10 and terran 5 $\to$ zerg 10.}
    \label{fig:transfer_learning}
\end{figure}

As shown in Table~\ref{tab:algorithms_specifications}, SPECTra-QMIX$^+$ is highly scalable, allowing it to adapt to larger environments without requiring additional parameters. This scalability makes it particularly well-suited for curriculum learning in settings with dynamically changing agent populations. While previous models like UPDeT have significantly improved transferability at the agent network level, they have not been extended to non-linear mixing networks. 

Specifically, $30\%$ of the total training occurs in a 5 vs 5 environment before transitioning to the target evaluation environment for the remaining $70\%$. We denote the variant of our model for curriculum learning by SPECTra-QMIX$^+$CL. As shown in Figure~\ref{fig:curriculum_learning}, SPECTra-QMIX$^+$CL achieves superior performance compared to the standard SPECTra-QMIX$^+$, demonstrating that even a simple curriculum learning strategy can enhance transferability.

Moreover, to evaluate the relative transfer ability of SPECTra-VDN and SPECTra-QMIX$^+$, we perform transfer learning from 5 vs 5 to 10 vs 10 to see how well SPECTra, HPN, and UPDeT generalize policies. Moreover, to evaluate the relative transfer ability of SPECTra-VDN and SPECTra-QMIX$^+$, we perform transfer learning from 5 vs 5 to 10 vs 10 to see how well SPECTra, HPN, and UPDeT generalize policies. To ensure that they understand games and not just unit-based repeated actions, we perform homology-based transfer learning and heterology-based transfer learning. The results in Figure~\ref{fig:transfer_learning} show that SPECTra-VDN and SPECTra-QMIX$^+$ perform well on homologous and heterologous transfer learning. This means that SPECTra-VDN and SPECTra-QMIX$^+$ can fully utilize the pre-trained models, which justifies the experimental results in Figure~\ref{fig:curriculum_learning}. 

In the SMACv2 20 vs 20 environments, SPECTra-QMIX$^+$ performs well in each environment despite having only $24\%$ parameters of HPN-QMIX. On the other hand, UPDeT-QMIX using self-attention shows a sharp drop in learning performance as the number of entities increases, despite having $3.4$ times more parameters than SPECTra-QMIX$^+$. This suggests that self-attention-based computation cannot cope with the exponential growth of complexity as the number of entities increases. 

\subsection{Google Research Football}

\begin{figure}[ht!]
    \centering
    \includegraphics[width=0.9\linewidth]{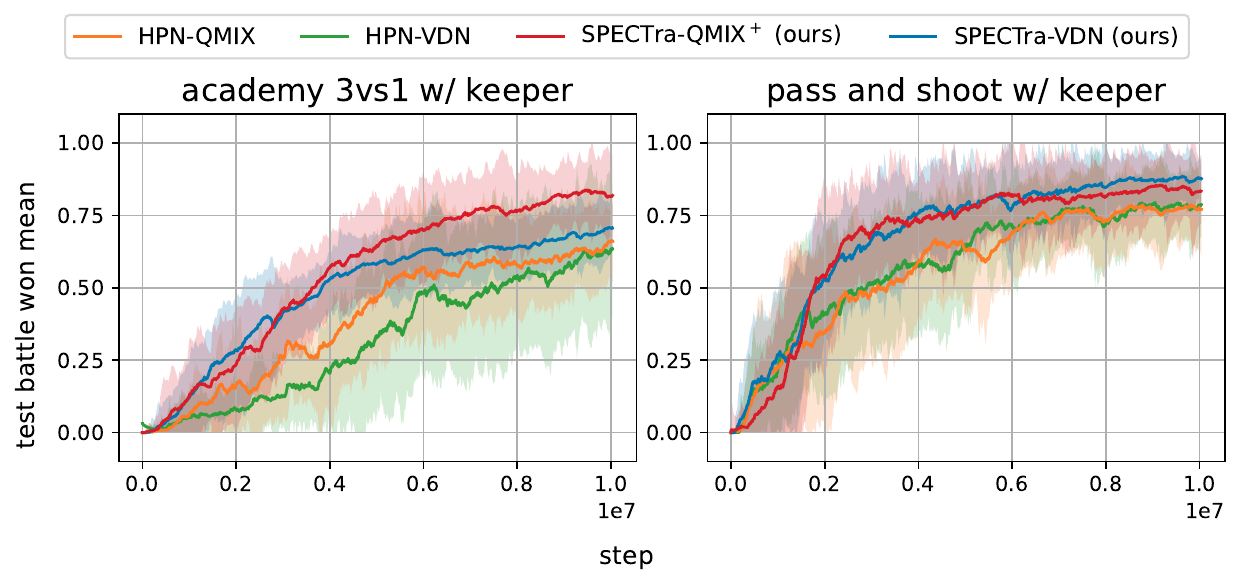}
    \caption{Comparison of SPECTra-QMIX$^+$ and SPECTra-VDN against baselines in Google Research Football.}
    
    \label{fig:GRF}
\end{figure}

To ensure the general performance of our algorithm, we conduct comparative experiments not only in the SMACv2 environment but also in the GRF environment. Unlike SMACv2, GRF controls all the agents of the home team except the goalkeeper and receive sparse team rewards only based on goal scoring. Moreover, addressing the permutation problem is crucial in GRF, as the vast action space leads to an enormous state space. To analyze the performance of SPECTra-QMIX$^+$ and SPECTra-VDN in GRF, we use HPN as the baseline, as HPN has demonstrated strong performance in this environment. 

Unlike SMACv2, GRF has a fixed action space, which requires different methods to solve the permutation problem. Since HPN does not currently provide an available implementation for GRF, we reproduce the algorithm as faithfully as possible based on the details available~\cite{hpn}. How we implemented algorithm in the GRF environment is described in detail in the Appendix~\ref{sec:grf_model}. 
To make SPECTra-QMIX$^+$ and SPECTra-VDN applicable to GRF, we do not use policy decoupling and apply mean pooling across agents for pass actions, which is problematic for the permutation problem, respectively. The experimental results in Figure~\ref{fig:GRF} demonstrate that both SPECTra-QMIX$^+$ and SPECTra-VDN achieve strong learning performance in GRF and further prove that SPECTra algorithms generally exhibit robust and effective performance in various MARL environments. The full experimental results can be found in the Appendix~\ref{sec:full_grf}.

\subsection{Ablation Study}
\label{sec:ablation_study}

We conduct ablation studies to evaluate the performance of each proposed component individually in SMACv2 environment. The full experimental results can be found in the Appendix~\ref{sec:ablation_appendix}.

\begin{figure}[ht!]
    \captionsetup{aboveskip=2pt, belowskip=0pt}
    \includegraphics[width=0.9\linewidth]{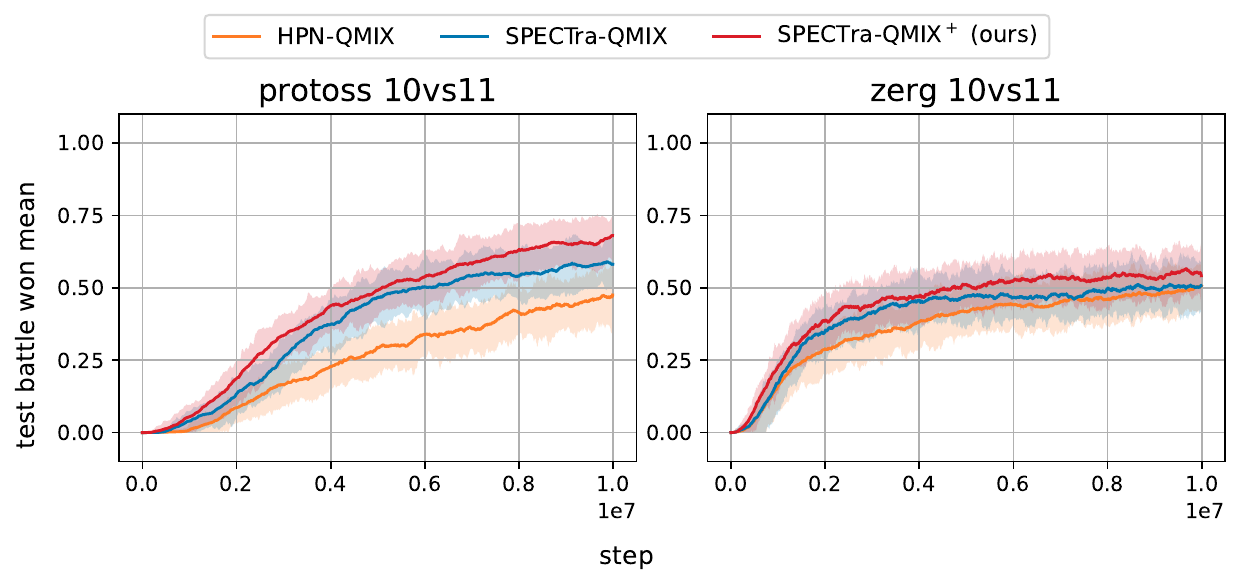}
    \caption{Experiments for the ablations of SAQA and ST-HyperNet modules.}
    
    \label{fig:main_ablation}
\end{figure}

\subsubsection{Main Ablations}
\label{sec:SAQA_HyperNet}

First, we analyze the effectiveness of the proposed agent networks and mixing networks in terms of performance. To evaluate the effectiveness of each module, we conduct an ablation study in a 10 vs 11 scenario across all races in SMACv2 (Protoss, Terran, and Zerg) using SPECTra-QMIX$^+$, HPN-QMIX, and SPECTra-QMIX, where SPECTra-QMIX is a model that changes the mixer used in HPN-QMIX to ST-HyperNet. The experimental results are presented in Figure~\ref{fig:main_ablation}.

By comparing the experimental results of SPECTra-QMIX and HPN-QMIX, we confirm that our agent networks outperform the agent networks in the original HPN model. This supports our claim that performing attention on each entity's observation embeddings is more effective than mean pooling. Furthermore, the comparison between SPECTra-QMIX$^+$ and SPECTra-QMIX demonstrates that the permutation-invariance property of our non-linear mixing network significantly contributes to improving model performance.

\begin{figure}[ht!]
    \captionsetup{aboveskip=2pt, belowskip=0pt}
    \includegraphics[width=1.\linewidth]{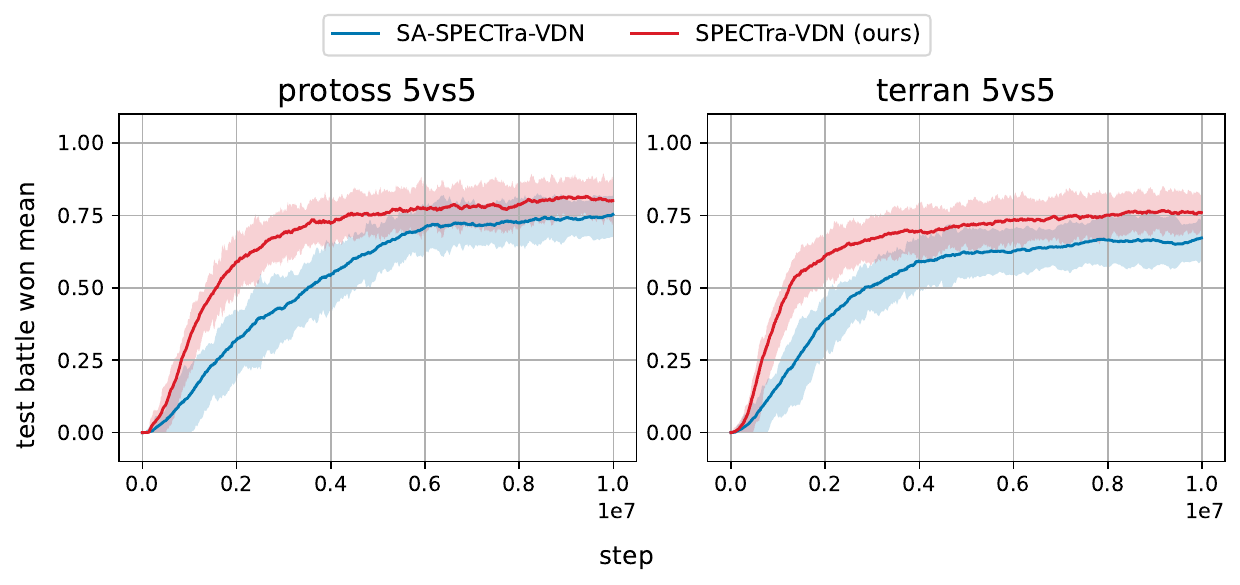}
    \caption{Comparison of SAQA and self-attention layers with mean pooling.}
    \label{fig:sa_vs_saqa_2races}
\end{figure}

\begin{figure}[ht!]
\centering
    \hspace{-0.4cm}
    \captionsetup{aboveskip=2pt, belowskip=2pt}
    \includegraphics[width=1.\linewidth,trim={0 0 12cm 0.3cm},clip]{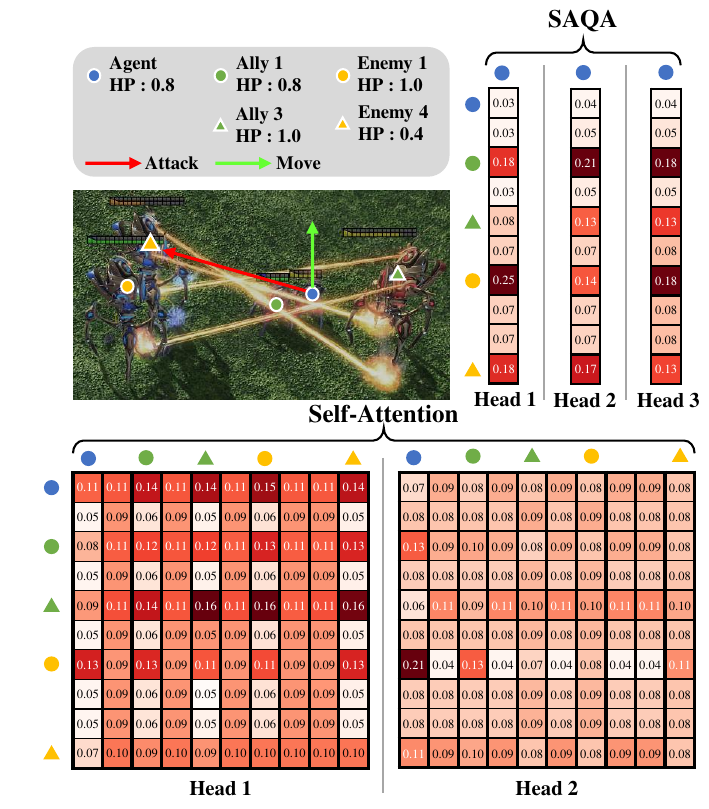}
    \caption{Example of attention maps computed from SAQA and self-attention layers (UPDeT) in SMACv2 environment.}
    
    \label{fig:attention_map}
\end{figure}

\subsubsection{SAQA vs Self-Attention}
\label{sec:sa_saqa}

Recall that we propose SAQA to address the high computational cost and sample complexity of conventional self-attention. Here, we examine whether SAQA achieves superior performance compared to self-attention in terms of computational efficiency and performance. To evaluate the architectural differences between these two modules, we conduct experiments in a 5 vs 5 scenario across all races using SPECTra-VDN and SA-SPECTra-VDN, where SA-SPECTra-VDN is a model derived from SPECTra-VDN by replacing SAQA with a self-attention layer. Figure~\ref{fig:sa_vs_saqa_2races} shows the experimental results.

Note that we use all observations along with the current agent information as queries, and the output embeddings of the self-attention are aggregated using a mean pooling operation. The experimental results demonstrate that self-attention unnecessarily increases sample complexity, whereas SAQA achieves higher performance with significantly lower computational cost.

We have conducted a qualitative analysis to examine how SAQA and self-attention layers interpret input observations differently. Figure~\ref{fig:attention_map} compares the attention scores of SAQA and self-attention when an observation from SMACv2 is input into \spectra and UPDeT. To analyze the response to entities that were not observed and thus processed as zero vectors, we trained the models without applying masks to any attention. While \spectra selected the attack action, UPDeT chose the move action. In the given state, the optimal action of the agent is to attack and eliminate Enemy 4 as quickly as possible since the enemy entity has significantly lower HP compared to Enemy 1. 
SAQA utilizes residual connections, allowing it to pay less attention to the query’s own information. This enables each head to focus on different aspects of the four observed objects. This suggests that all heads clearly understand the subject of the observation, indicating that the model can learn a unidirectional relationship centered on the observer. In contrast, UPDeT employs attention by using each entity's vector as the basis for selecting an action corresponding to that entity. The most crucial information for selecting an action should be the relationship between the observer agent and the given entity. To compute the action value for Enemy 4, the model should assign a high attention score to the relationship between the agent and Enemy 4. However, UPDeT tends to focus more on Ally 3 than on Enemy 4. This suggests that UPDeT struggles to prioritize the complex relationships among entities.

\section{Conclusions}

In this work, we have proposed a MARL framework called the \spectra that leverages Transformer-based neural architecture to address the scalability and permutation problems. By utilizing SAQA, we have mitigated the high computational cost of self-attention layers while maintaining the ability to learn complex interactions between many entities. Our experimental results confirm that SAQA accelerates the training process and achieves superior performance compared to baseline approaches. Additionally, we have successfully addressed the scalability and permutation problems in non-linear mixing networks by introducing the Set Transformer-inspired hypernetwork called the ST-HyperNet.

Furthermore, the \spectra framework has demonstrated strong applicability and generality by the results of transfer learning and curriculum learning. Our study presents an efficient way to adopt the well-known attention mechanism in highly complex MARL scenarios, and it holds the potential to be extended to real-world scale environments that require the capability to model the interactions among a significant number of agents.

\bibliographystyle{ACM-Reference-Format}
\bibliography{marl}

\newpage
\onecolumn
\appendix

\section{Additional Experiments}

\subsection{Comparison of Inference Time}
\label{sec:inference_time_comparison}
Figure~\ref{fig:inference_time} compares the inference time of each algorithm with 1,000 samples and shows that SAQA requires significantly less computation time than self-attention, supporting the results in Figure~\ref{fig:wide}. It is faster than UPDeT and outperforms HPN, and we confirm that this trend is consistent as the number of agents increases. 

\begin{figure}[ht!]
    \centering
    \includegraphics[width=1.\linewidth]{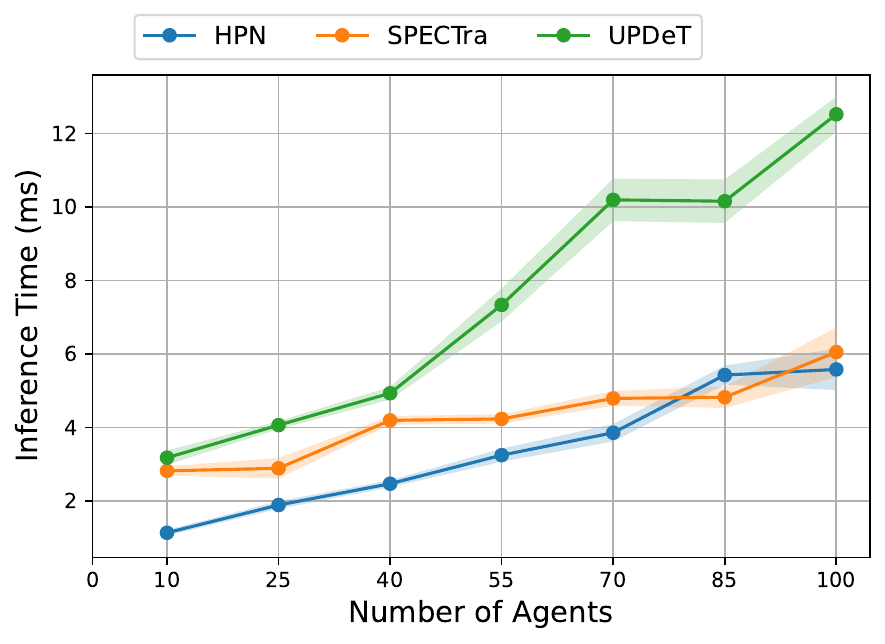}
    \caption{Comparison of inference time between the proposed \spectra and baseline methods.
    }
    
    \label{fig:inference_time}
\end{figure}

\subsection{SMACv2 Experiments}
\label{sec:full_smacv2}
Figure~\ref{fig:full_main_experiment} compares SPECTRa and other baseline algorithms in SMACv2 environment. SPECTRA-QMIX$^+$ shows outstanding performance in overall experiments.
\label{sec:smacv2_additional_experiments}

\begin{figure}[ht!]
    \includegraphics[width=1.\linewidth]{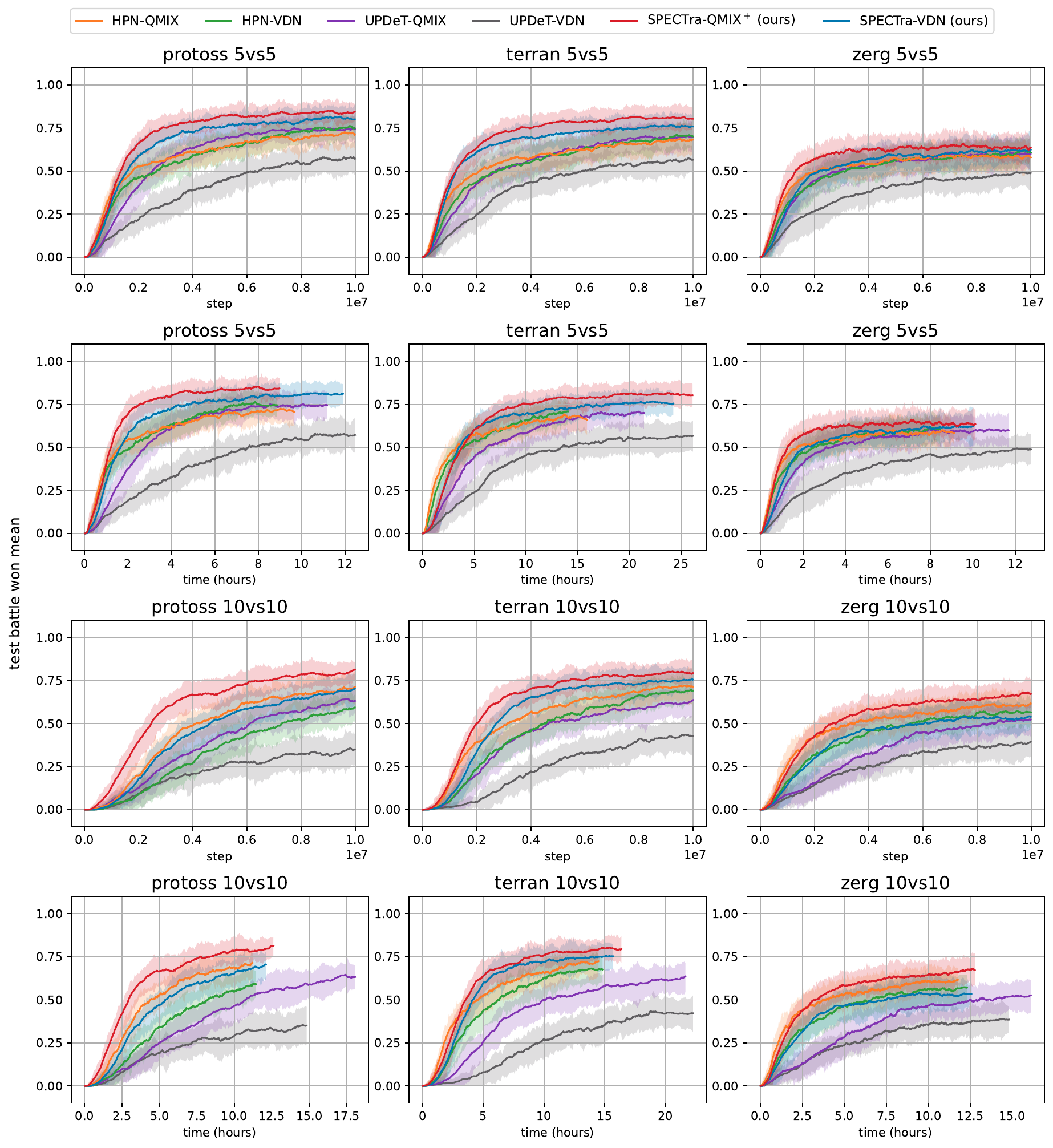}
    \includegraphics[width=1.\linewidth]{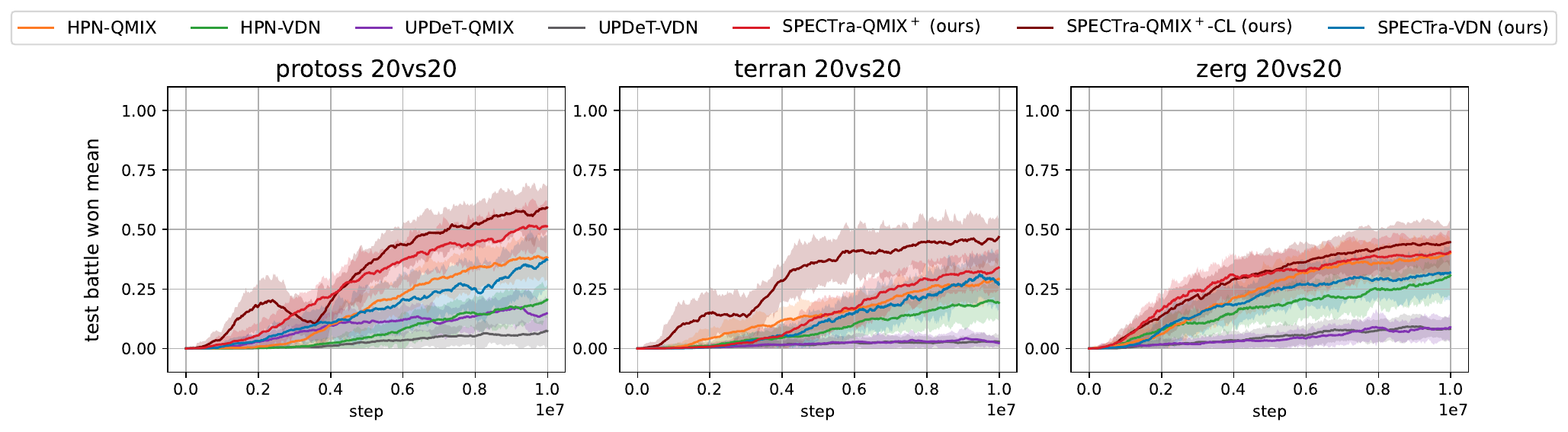}
    
    \caption{Overall comparison over SMACv2 challenges 5 vs 5, 10 vs 10, 20 vs 20.
    }
    \label{fig:full_main_experiment}
    
\end{figure}

\subsection{GRF Experiments}
\label{sec:full_grf}
Figure~\ref{fig:full_grf_experiments} shows the comparison between SPECTra and HPN. Each of them is combined to QMIX or VDN. Used scenarios in the figure are \textit{academy 3 vs 1 with keeper}, \textit{academy pass and shoot with keeper}, and \textit{academy run pass and shoot with keeper}, respectively.

\begin{figure}[ht!]
    \includegraphics[width=1.\linewidth]{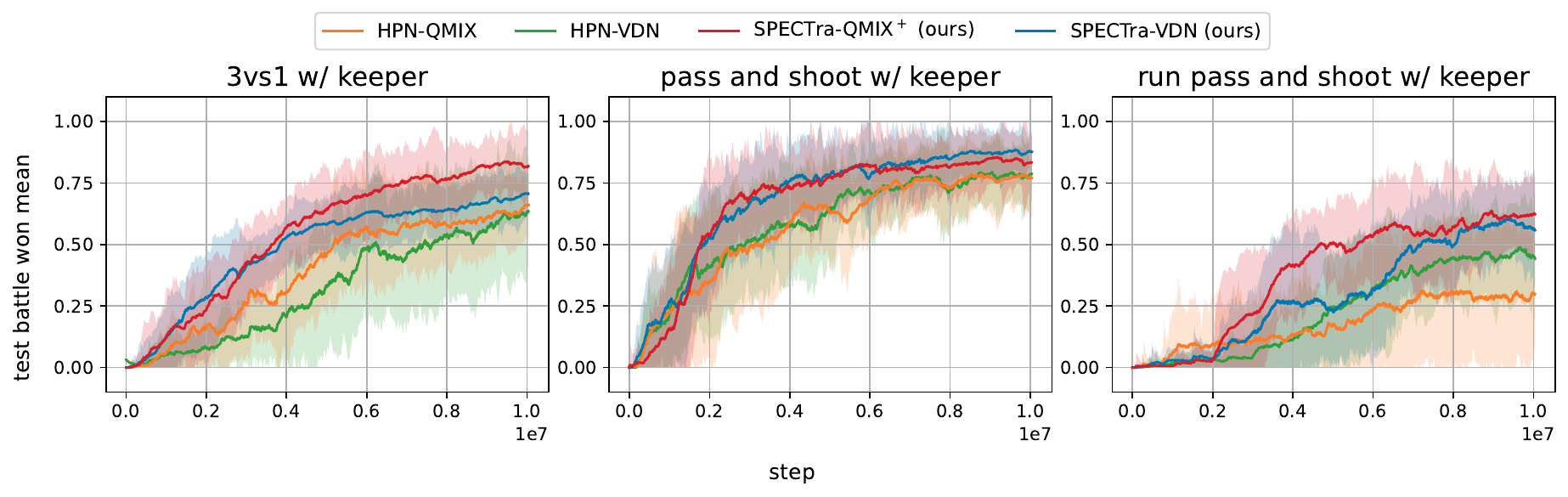}
    \caption{
    Comparison of SPECTra-QMIX$^+$ and SPECTra-VDN at the 3 maps against baselines in Google Research Football. 
    }
    \label{fig:full_grf_experiments}
\end{figure}

\subsection{Ablation Experiments}
\label{sec:ablation_appendix}
To understand how each module of the proposed algorithm contributes to performance, we conducted two ablation studies. Figure~\ref{fig:main_ablation} confirms that the proposed agent network and mixing network contribute to performance improvement. Additionally, Figure~\ref{fig:sa_vs_saqa} shows that SAQA is more effective than embedding generation using conventional self-attention and directly impacts performance.
\begin{figure}[h]
    \includegraphics[width=1.\linewidth]{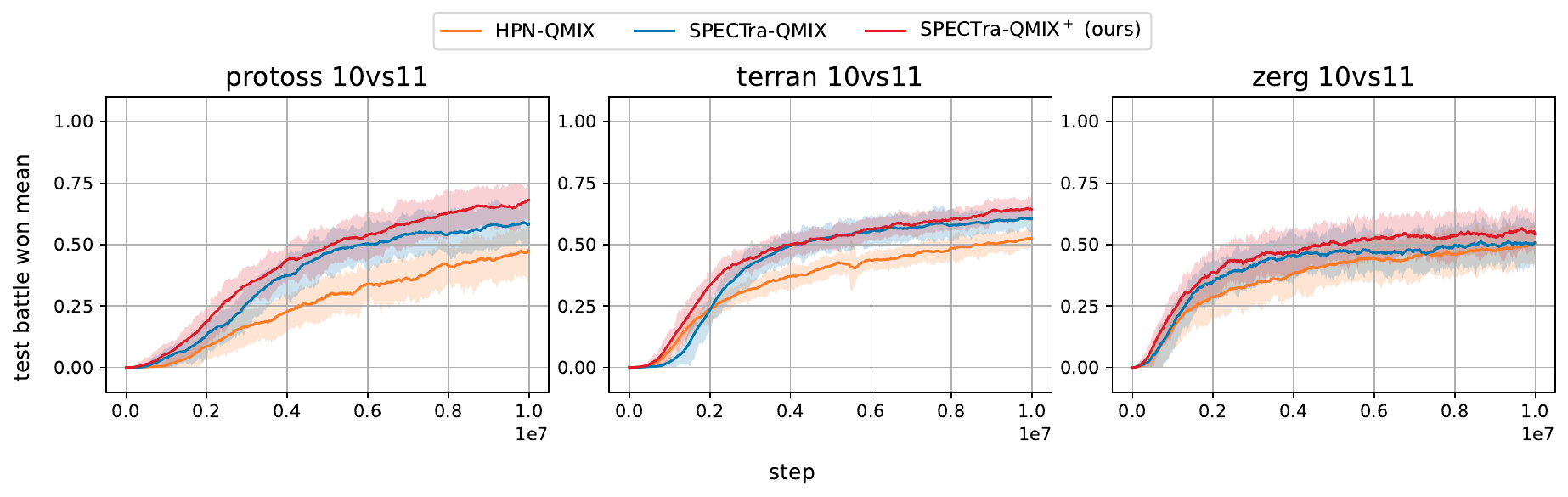}
    \caption{Ablation experiments for SAQA and ST-HyperNet modules.}
    
    \label{fig:main_ablation}
\end{figure}

\begin{figure}[h]
    \includegraphics[width=1.\linewidth]{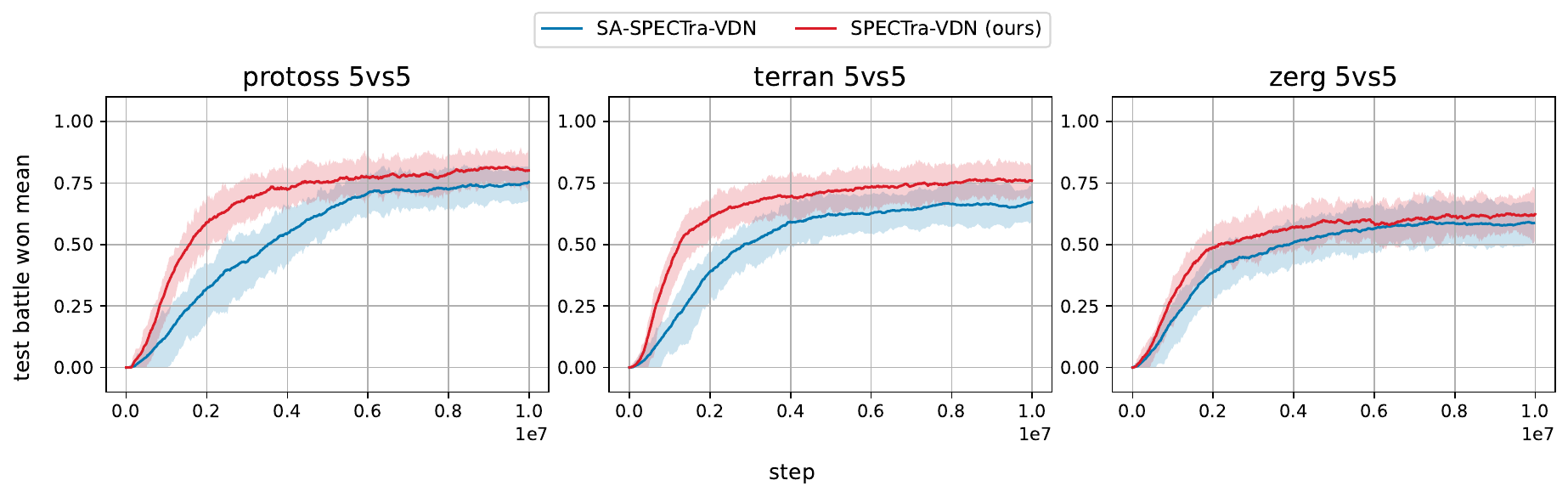}
    \caption{Comparison of SAQA and self-attention layers with mean pooling in 3 races.
    }
    \label{fig:sa_vs_saqa}
   
\end{figure}

\section{Hyperparameter}
\label{sec:hyperparameter}
During experiments, the implementations of baseline methods are consistent with their official
repositories. We list the detailed hyperparameter settings we used in our SMACv2, GRF, and Ablation study in the paper in the Table~\ref{tab:smacv2_hyperparameters} $\sim$ \ref{tab:ablation_SA_SAQA_hyper} below. 

\subsection{Hyperparameters in SMACv2 Experiments}
\label{sec:smacv2_hyperparameter}
In SMACv2, the batch\_size is set differently depending on the total number of entities to manage memory stably. Specifically, 128 for 5 vs 5 environments, 32 for 10vs10, and 10 for 20vs20. In addition, SPECTra-QMIX$^+$ uses the attention scores of ST-HyperNet as the weight and bias of the mixing network, and these values change dynamically with the number of agents. Therefore, the number of agents dynamically determines the mixing\_embed\_dim.

\begin{table}[ht!]
    \centering
    \resizebox{\textwidth}{!}{
        \begin{tabular}{l cccccc}
            \toprule
            \bf{Hyperparameter}  & \bf{SPECTra-QMIX$^+$} & \bf{HPN-QMIX}  & \bf{UPDeT-QMIX} & \bf{SPECTra-VDN} & \bf{HPN-VDN} & \bf{UPDeT-VDN}\\ \midrule
            mac    & ss\_mac & hpn\_mac & updet\_mac & ss\_mac & hpn\_mac & updet\_mac\\  
            agent  & ss\_rnn & hpns\_rnn & updet\_agent & ss\_rnn & hpns\_rnn & updet\_agent\\
            learner & nq\_learner & nq\_learner & nq\_learner & nq\_learner & nq\_learner & nq\_learner \\
            mixer & ss\_mixer & qmix & qmix & vdn & vdn & vdn \\
            hidden\_size & 64 & 64 & 32 & 64 & 64 & 32 \\
            n\_head & 4 & - & 3 & 4 & - & 3\\
            transformer\_depth & 1 & - & 2 & 1 & - & 2 \\
            mixing\_n\_head & 4 & - & - & - & - & -\\
            mixing\_embed\_dim  & number of agents & 32 & 32 & - & - & -\\
            hypernet\_embed  & 32 & 64 & 64 & - & - & -\\
            batch\_size & 128 & 128 & 128 & 128 & 128 & 128\\
            \bottomrule
        \end{tabular}
    }
    \caption{Hyperparameter Settings of SMACv2 Experiment Algorithms.}
    \label{tab:smacv2_hyperparameters}
\end{table}
\subsection{Hyperparameters in GRF Experiments}
GRF is one of the most challenging environments for MARL benchmarks due to its score-based sparse rewards and long-term episodes. If we use only the score-based reward function of the environment, the over-exploration will significantly degrade training speed and performance. So, it is necessary to customize and override the reward function to prevent over-exploration. Also, since the action space differs depending on the environment, each model should use an optimized action sampling method for solving the permutation problem according to the environment. Since implementing this part of the HPN~\cite{hpn} paper is not publicly available in code, we designed a GRF training based on the GitHub project published by CDS~\cite{cds}. We implemented a model that follows the architecture described in the HPN~\cite{hpn} paper. We also tried accelerating training by using additional action masks and reward functions. All implementations are 100\% publicly available on our GitHub project. Below is the implementation of the hyperparameter table we used to train the model: \ref{tab:grf_hyperparameters}. 

\begin{table}[ht!]
    \centering
    \resizebox{\textwidth}{!}{
        \begin{tabular}{l cccc}
            \toprule
            \bf{Hyperparameter}  & \bf{SPECTra-QMIX$^+$} & \bf{HPN-QMIX} & \bf{SPECTra-VDN} & \bf{HPN-VDN}  \\ \midrule
            mac    & ss\_mac & hpn\_mac  & ss\_mac & hpn\_mac \\  
            agent  & ss\_rnn & hpns\_rnn  & ss\_rnn & hpns\_rnn \\
            learner & nq\_learner & nq\_learner & nq\_learner & nq\_learner \\
            mixer & ss\_mixer & qmix  & vdn & vdn \\
            hidden\_size & 64 & 64 & 64 & 64 \\
            n\_head & 4 & - & 4 & - \\
            transformer\_depth & 1 & -  & 1 & - \\
            mixing\_n\_head & 4 & -  & - & -\\
            mixing\_embed\_dim  & - & 32  & - & -\\
            hypernet\_embed  & 32 & 64  & - & - \\
            batch\_size & 32 & 32 & 32  & 32\\
            \bottomrule
        \end{tabular}
    \label{tab:grf_hyperparameters}
    }
    \caption{Hyperparameter Settings of Google Research Football Experiment.}
\end{table}

\subsection{Hyperparameters in Ablation Studies}
\label{sec:ablation_hyperparameter}
We have conducted ablation studies to evaluate the performance of each proposed component individually. As discussed in Section~\ref{sec:ablation_study}, each ablation study demonstrates that the proposed components contribute to the performance of our model in a positive manner. Unless otherwise mentioned, the hyperparameters used in the experiments follow the values provided in Appendix~\ref{sec:smacv2_hyperparameter}. The hyperparameter settings for the ablation studies are presented in Tables~\ref{tab:ablation_SAQA_HyperNet_hyper} to \ref{tab:ablation_SA_SAQA_hyper}.

\begin{table}[ht!]
    \centering
    \resizebox{\textwidth}{!}{ 
    \begin{tabular}{l ccc}
        \toprule
        \bf{Hyperparameter}       & \bf{SPECTra-QMIX$^+$} & \bf{SPECTra-QMIX} & \bf{HPN-QMIX}  \\ 
        \midrule
        mixer                     & ss\_mixer     & qmix              & qmix           \\
        n\_head                   & 4             & 4                 & -              \\
        transformer\_depth        & 1             & 1                 & -              \\
        mixing\_n\_head           & 1             & -                 & -              \\
        mixing\_embed\_dim        & -             & 32                & -              \\
        hypernet\_embed           & 32            & 64                & -              \\
        batch\_size               & 32            & 32                & 32             \\
        \bottomrule
    \end{tabular}
    \label{tab:ablation_SAQA_HyperNet_hyper}
    }
    \caption{Hyperparameter Settings of Ablation~\ref{sec:SAQA_HyperNet} Experiment Algorithms.}
\end{table}

\begin{table}[ht!]
    \centering
    \resizebox{\textwidth}{!}{
    \begin{tabular}{l cc}
        \toprule
        \bf{Hyperparameter}  & \bf{SPECTra-VDN} & \bf{SA-SPECTra-VDN} \\ 
        \midrule
        mixer               & vdn  & vdn  \\
        n\_head            & 4    & 4    \\
        transformer\_depth & 1    & 1    \\
        batch\_size        & 128  & 128  \\
        use\_SAQA          & \text{True}  & \text{False} \\
        \bottomrule
    \end{tabular}
    }
    \caption{Hyperparameter Settings of Ablation~\ref{sec:sa_saqa} Experiment Algorithms.}
    \label{tab:ablation_SA_SAQA_hyper}
\end{table}

\section{Model Parameters}
\label{sec:model_parameters}

In this section, we compare the number of parameters in the models used in our experiments.

\subsection{Model Parameters in SMACv2 Experiments}
\label{sec:smacv2_parameters}

In the Protoss scenario, ``shield" features are added to the feature information, resulting in different feature dimensions compared to scenarios of other races, which leads to a different number of parameters. However, since Terran and Zerg share the same feature dimensions, their architectures have the same number of parameters. As shown in Table~\ref{tab:smacv2_parameters}, both SPECTra-QMIX$^+$ and SPECTra-VDN maintain the same number of parameters regardless of the number of agents and enemies, making them scalable. In contrast, HPN-QMIX and HPN-VDN are not scalable, as the architecture of HPN is dependent on the number of agents. For UPDeT, while UPDeT-VDN remains scalable as we simply add $Q$-values in VDN, UPDeT-QMIX is not scalable as its mixing network is still dependent on the number of agents and its number of parameters also varies with the number of entities.

\begin{table}[ht!]
\label{sec:number-of-parameters}
    \centering
    \resizebox{\textwidth}{!}{

        \begin{tabular}{cccccccc}
            \toprule
            \bf{Race} &\bf{Map}  & \bf{SPECTra-QMIX$^+$} & \bf{HPN-QMIX}  & \bf{UPDeT-QMIX} & \bf{SPECTra-VDN} & \bf{HPN-VDN} & \bf{UPDeT-VDN}\\ 
            \midrule
            \multirow{3}{*}{\centering Protoss}
            & 5 vs 5   & 72.90 & 144.488 & 80.903  & 48.326 & 106.823 & 43.238 \\
            & 10 vs 10 & 72.90 & 189.768 & 125.863 & 48.326 & 107.143 & 43.238 \\
            & 20 vs 20 & 72.90 & 309.128 & 244.583 & 48.326 & 107.783 & 43.238 \\
            \hdashline
            \multirow{3}{*}{\centering Terran} 
            & 5 vs 5   & 72.582 & 134.056 & 78.951  & 48.134 & 98.311  & 43.206 \\
            & 10 vs 10 & 72.582 & 177.416 & 121.991 & 48.134 & 98.631  & 43.206 \\
            & 20 vs 20 & 72.582 & 292.936 & 236.871 & 48.134 & 99.271  & 43.206 \\
            \hdashline
            \multirow{3}{*}{\centering Zerg} 
            & 5 vs 5   & 72.582 & 134.056 & 78.951  & 48.134 & 98.311  & 43.206 \\
            & 10 vs 10 & 72.582 & 177.416 & 121.991 & 48.134 & 98.631  & 43.206 \\
            & 20 vs 20 & 72.582 & 292.936 & 236.871 & 48.134 & 99.271  & 43.206 \\
            \bottomrule
        \end{tabular}
    
    }
    \label{tab:smacv2_parameters}
    \caption{Number of parameters each algorithm has according to the SMACv2 maps (in thousands, k).}
\end{table}

\subsection{Model Parameters in GRF Experiments}
\label{sec:grf_parameters}

Due to differences in action space, GRF and SMACv2 require separately designed layers for estimating action values. To ensure effective training, we adapt the models as described in Appendix~\ref{sec:grf_model}. Table~\ref{tab:smacv2_parameters} below presents the number of parameters for the models implemented in GRF.

\begin{table}[ht!]
\label{sec:number-of-parameters}
    \centering
    \resizebox{\textwidth}{!}{ 
        \begin{tabular}{ccccc}
            \toprule
            \bf{Scenario} & \bf{SPECTra-QMIX$^+$} & \bf{HPN-QMIX} & \bf{SPECTra-VDN} & \bf{HPN-VDN} \\ 
            \midrule
            \centering 3 vs 1 w/ keeper
            & 72.844 & 86.901 & 56.787 & 73.364 \\
            \hdashline
            \centering pass and shoot w/ keeper
            & 72.844 & 84.053 & 56.787 & 73.364 \\
            \hdashline
            \centering run pass and shoot w/ keeper
            & 72.844 & 84.053 & 56.787 & 73.364 \\
            \bottomrule
        \end{tabular}
    }
    \caption{Number of parameters each algorithm has according to the GRF scenarios (in thousands, k).}
    \label{tab:grf_parameters}
\end{table}

\section{Model Adaptations for Google Research Football}
\label{sec:grf_model}
This section describes how the experimental algorithms were implemented in the GRF benchmark. Unlike SMACv2, where an action is performed by identifying a specific entity, GRF does not allow agents to control which ally to pass the ball to directly. The GRF has a fixed number of actions, so we must change how the agent network outputs actions to address scalability and permutation issues. First, to implement HPN-QMIX and HPN-VDN, we need to consider the action permutations. GRF has 19 actions, including permutation-invariant actions such as moving, sliding, and shooting, and permutation-equivariant actions such as long pass, high pass, and short pass. We apply the max pooling operation to all friendly-related action values to obtain the final q-values for the three pass actions. To implement SPECTra-QMIX$^+$ and SPECTra-VDN, we remove the policy decoupling modules and apply the mean pooling operation for the three pass actions to obtain their final q-values. The above implementation is 100\% public.
\section{Computation Resource}

Experiments were conducted on the Intel(R) Xeon(R) Gold 6240R CPU @ 2.40GHz 24 core processor with 768GB RAM and NVIDIA A10 D6 24GB.

\end{document}